\documentclass[runningheads]{llncs}
\usepackage{graphicx}
\usepackage{booktabs} 
\usepackage[usenames,dvipsnames]{color}
\usepackage{amsmath}
\newtheorem{assumption}{Assumption}
\usepackage{amssymb}
\usepackage{url}
\usepackage{nicefrac}
\usepackage{algorithm}
\usepackage[noend]{algpseudocode}
\usepackage{bbold} 			
\usepackage{filecontents}
\usepackage{pgfplots}
\usepgfplotslibrary{colormaps}
\usepgfplotslibrary{groupplots}
\pgfplotsset{compat=newest}
\usepackage[misc]{ifsym}

\DeclareMathOperator*{\tr}{tr}
\DeclareMathOperator*{\diag}{diag}

\newcommand{\noop}[1]{}

\makeatletter
\newcommand{\thickbar}{\mathpalette\@thickbar}
\newcommand{\@thickbar}[2]{{#1\mkern1.5mu\vbox{
  \sbox\z@{$#1\mkern-1.5mu#2\mkern-1.5mu$}%
  \sbox\tw@{$#1\overline{#2}$}%
  \dimen@=\dimexpr\ht\tw@-\ht\z@-.8\p@\relax
  \hrule\@height.8\p@ 
  \vskip\dimen@
  \box\z@}\mkern1.5mu}
}
\makeatother
\definecolor{cSPH}{RGB}{241, 140, 157}
\definecolor{cSTSC}{RGB}{152, 218, 198}
\definecolor{cPG}{RGB}{91, 199, 171}
\definecolor{cDIP}{RGB}{229, 214, 41}

\begin{document}
%
\title{$k$ is the Magic Number --- Inferring the Number of Clusters Through Nonparametric Concentration Inequalities}
\titlerunning{Special $k$}
\toctitle{$k$ is the Magic Number --- Inferring the Number of Clusters Through Nonparametric Concentration Inequalities}
%
\author{Sibylle Hess (\Letter) \and Wouter Duivesteijn}
\authorrunning{S. Hess and W. Duivesteijn}
\tocauthor{Sibylle Hess (Technische Universiteit Eindhoven) and Wouter Duivesteijn (Technische Universiteit Eindhoven)}
%
\institute{Data Mining Group, Technische Universiteit Eindhoven, Eindhoven, the Netherlands
\email{\{s.c.hess,w.duivesteijn\}@tue.nl}}
\maketitle              
\begin{abstract}
Most convex and nonconvex clustering algorithms come with one crucial parameter: the $k$ in $k$-means. 
To this day, there is not one generally accepted way to accurately determine this parameter. Popular methods are simple yet theoretically unfounded, such as searching for an elbow in the curve of a given cost measure. In contrast, statistically founded methods often make strict assumptions over the data distribution or come with their own optimization scheme for the clustering objective. This limits either the set of applicable datasets or clustering algorithms. 
In this paper, we strive to determine the number of clusters by answering a simple question: given two clusters, is it likely that they jointly stem from a single distribution? To this end, we propose a bound on the probability that two clusters originate from the distribution of the unified cluster, specified only by the sample mean and variance. Our method is applicable as a simple wrapper to the result of any clustering method minimizing the objective of $k$-means, which includes Gaussian mixtures and Spectral Clustering. We focus in our experimental evaluation on an application for nonconvex clustering and demonstrate the suitability of our theoretical results.  Our \textsc{SpecialK} clustering algorithm automatically determines the appropriate value for $k$, without requiring any data transformation or projection, and without assumptions on the data distribution. Additionally, it is capable to decide that the data consists of only a single cluster, which many existing algorithms cannot.

\end{abstract}
%
%
%
\section{Introduction}
When creating a solution to the task of clustering ---finding a natural partitioning of the records of a dataset into $k$ groups--- the holy grail is to automatically determine $k$.  The current state of the art in clustering research has not yet achieved this holy grail to a satisfactory degree.  Since many papers on parameter-free clustering exist, this statement might sound unnecessarily polemic without further elaboration. Hence, we must describe what constitutes ``satisfactory'', which we will do by describing unsatisfactory aspects of otherwise perfectly fine (and often actually quite interesting) clustering solutions. Some solutions can only handle convex cluster shapes \cite{1982Lloyd}, while most real-life phenomena are not necessarily convex.
Some solutions manage to avoid determining $k$, at the cost of having to specify another parameter that effectively controls $k$ \cite{2010Alamgir}. Some solutions define a cluster criterion and an algorithm to iteratively mine the data: the best single cluster is found, after which the algorithm is run again on the data minus that cluster, etcetera \cite{2014Hou}; this runs the risk of finding a local optimum.  We, by contrast, introduce a clustering algorithm that can handle nonconvex shapes, is devoid of parameters that demand the user to directly or indirectly set $k$, and finds the global optimum for its optimization criterion.


We propose a probability bound on the operator norm of centered, symmetric decompositions based on the matrix Bernstein concentration inequality. We apply this bound to assess whether two given clusters are likely to stem from the distribution of the unified cluster. Our bound provides a statistically founded decision criterion over the minimum similarity within one cluster and the maximum similarity between two clusters: this entails judgment on whether two clusters should be separate or unified. Our method is easy to implement and statistically nonparametric. Applied on spectral clustering methods, to the best of the authors' knowledge, providing a statistically founded way to automatically determine $k$ is entirely new.

We incorporate our bound in an algorithm called \textsc{SpecialK}, since it provides a method for SPEctral Clustering to Infer the Appropriate Level $k$.  On synthetic datasets, \textsc{SpecialK} outperforms some competitors, while performing roughly on par with another competitor.  However, when unleashed on a synthetic dataset consisting of just random noise, all competitors detect all kinds of clusters, while only \textsc{SpecialK} correctly determines the value of $k$ to be one. If you need an algorithm that can correctly identify a sea of noise, \textsc{SpecialK} is the only choice.  On four real-life datasets with class labels associated to the data points, we illustrate how all available algorithms, including \textsc{SpecialK}, often cannot correctly determine the value of $k$ corresponding to the number of classes available in the dataset.  We argue that this is an artefact of a methodological mismatch: the class labels indicate one specific natural grouping of the points in the dataset, but the task of clustering is to retrieve \emph{any} natural grouping of the points in the dataset, not necessarily the one encoded in the class label.  Hence, such an evaluation is fundamentally unfit to assess the performance of clustering methods. 

\section{Three Sides of a Coin: $k$-means, Gaussian Mixtures and Spectral Clustering}
To embed our work in already existing related work, and to describe the preliminary building blocks required as foundation on which to build our work, we must first introduce some notation. We write $\mathbf{1}$ for a constant vector of ones, whose dimension can be derived from context unless otherwise specified. We denote with $\mathbb{1}^{m\times k}$ the set of all binary matrices which indicate a partition of $m$ elements into $k$ sets. 
Such a partition is computed by $k$-means clustering; every element belongs to exactly one cluster.  Let $D\in\mathbb{R}^{m\times n}$ be a data matrix, collecting $m$ points $D_{j\cdot}$, which we identify with their index $j$. The objective of $k$-means is equivalent to solving the following matrix factorization problem:
\begin{align}
    \min_{Y,X}\ \left\|D-YX^\top\right\|^2 &\quad\text{s.t.}\quad Y\in\mathbb{1}^{m\times k}, X\in\mathbb{R}^{n\times k}.\label{eq:kmeans}
\end{align}
The matrix $Y$ indicates the cluster assignments; point $j$ is in cluster $c$ if $Y_{jc}=1$. The matrix $X$ represents the cluster centers, which are given in matrix notation as $X=D^\top Y\left(Y^\top Y\right)^{-1}$. The well-known optimization scheme of $k$-means, \emph{Lloyd's algorithm} \cite{1982Lloyd}, employs the convexity of the $k$-means problem if one of the matrices $X$ or $Y$ is fixed. The algorithm performs an alternating minimization, updating $Y$ to assign each point to the cluster with the nearest center, and updating $X_{\cdot c}$ as the mean of all points assigned to cluster $c$.
\subsection{Gaussian Mixtures}
The updates of Lloyd's algorithm correspond to the expectation and maximization steps of the EM-algorithm~\cite{bishop2006pattern,bauckhage2015clustering}. In this probabilistic view, we assume that every data point $D_{j\cdot}$ is independently sampled 
by first selecting cluster $c$ with probability $\pi_c$, and then 
sampling point $D_{j\cdot}$ from the Gaussian distribution: 
    \[p(\xi|c)=\frac{1}{\sqrt{2\pi\epsilon}}\exp\left(-\frac{1}{2\epsilon}\left\|\xi-X_{\cdot c}\right\|^2\right)\sim \mathcal{N}(\xi|X_{\cdot c},\epsilon I).\]
This assumes that the covariance matrix of the Gaussian distribution is equal for all clusters: $\Sigma_c=\epsilon I$. 
From this sampling procedure, we compute the log-likelihood for the data and cluster assignments:
\begin{align*}
    \log p(D,Y|X,\epsilon I,\pi) &= \log\left(\prod_{j=1}^m\prod_{c=1}^k\left(\frac{\pi_c}{\sqrt{2\pi\epsilon}}\exp\left(-\frac{1}{2\epsilon}\left\|D_{j\cdot}-X_{\cdot c}^\top\right\|^2\right)\right)^{Y_{jc}}\right)\\
    &=\sum_{j=1}^m\sum_{c=1}^kY_{jc}\left(\ln\left(\frac{\pi_c}{\sqrt{2\pi\epsilon}}\right)-\frac{1}{2\epsilon}\left\|D_{j\cdot}-X_{\cdot c}^\top\right\|^2\right) \\
    &= -\frac{1}{2\epsilon}\left\|D-YX^\top\right\|^2-\frac{m}{2}\ln(2\pi\epsilon)+\sum_{c=1}^k|Y_{\cdot c}|\ln\left(\pi_c\right).
\end{align*}
Hence, if $\epsilon>0$ is small enough, maximizing the log-likelihood of the Gaussian mixture model is equivalent to solving the $k$-means problem.   
\subsection{Maximum Similarity versus Minimum Cut}
There are multiple alternative formulations of the $k$-means problem. 
Altering the Frobenius norm in Equation \eqref{eq:kmeans} with the identity $\|A\|^2=\tr(AA^\top)$ begets:
\begin{align*}
    \left\|D-YX^\top\right\|^2 &= \|D\|^2 - 2\tr\left(D^\top YX^\top\right) + \tr\left(XY^\top YX^\top\right)\\
    &= \|D\|^2 - \tr\left(Y^\top DD^\top Y\left(Y^\top Y\right)^{-1}\right),
\end{align*}
where the last equality derives from inserting the optimal $X$, given $Y$. Thus, we transform the $k$-means objective, defined in terms of distances to cluster centers, to an objective defined solely on similarity of data points. The matrix $DD^\top$ represents similarity between points, measured via the inner product $sim(j,l)=D_{j\cdot}D_{l\cdot}^\top$. The $k$-means objective in Equation \eqref{eq:kmeans} is thus equivalent to the \emph{maximum similarity problem} for a similarity matrix $W=DD^\top$:
\begin{equation}\label{eq:maxsimilarity}
    \max_{Y\in\mathbb{1}^{m\times k}}Sim(W,Y)=\sum_cR(W,Y_{\cdot c}),\quad R(W,y) = \frac{y^\top W y}{|y|}.
\end{equation}
Here, we introduce the function $R(W,y)$, which is known as the \emph{Rayleigh coefficient} \cite{1985Horn}, returning the ratio similarity of points within cluster $y$. 

An alternative to maximizing the \emph{similarity within} a cluster, is to minimize the \emph{similarity between} clusters. This is known as the \emph{ratio cut} problem, stated for a symmetric similarity matrix $W$ as:
\begin{equation*}
    \min_{Y\in\mathbb{1}^{m\times k}} Cut(W,Y)=\sum_s C(W,Y),\quad C(W,y)=\frac{y^\top W \thickbar{y}}{|y|}.
\end{equation*}
The function $C(W,y)$ sums the similarities between points indicated by cluster $y$ and the remaining points indicated by $\thickbar{y}=\mathbf{1}-y$. Imagining the similarity matrix $W$ as a weighted adjacency matrix of a graph, the function $C(W,y)$ sums the weights of the edges which would be cut if we \emph{cut out} the cluster $y$ from the graph. 
Defining the matrix $L=\diag(W\mathbf{1})-W$, also known as the \emph{difference graph Laplacian} \cite{1997Chung}, we have $C(W,y)=R(L,y)$.
As a result, the maximum similarity problem with respect to the similarity matrix $-L$ is equivalent to the minimum cut problem with similarity matrix $W$.
\subsection{Spectral Clustering}
\label{sec:spectralClust}
If similarities are defined via the inner product, then the similarity in Equation~\eqref{eq:maxsimilarity} is maximized when \emph{every} point in a cluster is similar to \emph{every other} point in that cluster. As a result, the obtained clusters by $k$-means have convex shapes. If we expect nonconvex cluster shapes, then our similarities should only locally be compared. This is possible, e.g., by defining the similarity matrix as the adjacency matrix to the $k$NN graph or the $\epsilon$-neighborhood graph. 
Clustering methods employing such similarities are known as \emph{spectral clustering} \cite{2001Ng}. It is related to minimizing the cut for the graph Laplacian of a given similarity matrix. Spectral clustering computes a truncated eigendecomposition of  $W=-L\approx V^{(k+1)}\Lambda^{(k+1)}{V^{(k+1)}}^\top$, where $\Lambda^{(k+1)}$ is a diagonal matrix having the $(k+1)$ largest eigenvalues on its diagonal $\Lambda_{11}\geq\ldots\geq \Lambda_{k+1 k+1}$, and $V^{(k+1)}$ represents the corresponding eigenvectors.
Graph Laplacian theory says that the eigenvectors to the largest eigenvalue indicate the connected components of the graph, while in practical clustering application the entire graph is assumed to be connected.  
To honor this assumption, the first eigenvector is omitted from the matrix $V^{(k+1)}$, which is subsequently discretized by $k$-means clustering~\cite{von2007tutorial}.   
Considering the relation between the minimum cut objective and $k$-means clustering, the objective to minimize the $Cut(L,Y)$  
is actually equivalent to solving $k$-means clustering for a data matrix $D$ such that $W=DD^\top$. This relation was recently examined~\cite{hess2019spectacl}, with the disillusioning result that $k$-means clustering on the decomposition matrix $D$ usually returns a local optimum, whose objective value is close to the global minimum but whose clustering is unfavorable. Consequently, the authors propose the algorithm \textsc{SpectACl}, approximating the similarity matrix by a matrix product of projected eigenvectors, such that:
\begin{equation}\label{eq:spectacl}
   W\approx DD^\top,\quad D_{ji}=\left|V^{(n)}_{ji}\right|\left|\Lambda^{(n)}_{ii}\right|^{-(1/2 )}
\end{equation} 
for a large enough dimensionality $n>k$. 
Although this increases the rank of the factorization from $k$ in traditional spectral clustering to $n>k$, the search space, which is spanned by the vectors $D_{\cdot i}$, is reduced in \textsc{SpectACl}. The projection of the orthogonal eigenvectors $V_{\cdot i}$ to the positive orthant introduces linear dependencies among the projections $D_{\cdot i}$.
\subsection{Estimating $k$}
Depending on the view on $k$-means clustering ---as a matrix factorization, a Gaussian mixture model, or a graph-cut algorithm--- we might define various strategies to derive the correct $k$. The \emph{elbow} strategy is arguably the most general approach. Plotting the value of the objective function for every model when increasing $k$, a kink in the curve is supposed to indicate the correct number of clusters. With regard to spectral clustering, the elbow method is usually deployed on the 
largest eigenvalues of the Laplacian, called \emph{eigengap heuristic}~\cite{von2007tutorial}. Depending on the application, the elbow may not be easy to spot, and the selection of $k$ boils down to a subjective trade-off between data approximation and model complexity.

To manage this trade-off in a less subjective manner, one can define a cost measure beforehand. Popular approaches employ Minimum Description Length (MDL)~\cite{2007bohm,2005kontkanen} or the Bayesian Information Criterion (BIC)~\cite{2000pelleg}. The nonconvex clustering method Self-Tuning Spectral Clustering (STSC)~\cite{zelnik2005self} defines such a cost measure on the basis of spectral properties of the graph Laplacian. The $k$-means discretization step is replaced by the minimization of this cost measure, resulting in a rotated eigenvector matrix which approximates the form of partitioning matrices, having only one nonzero entry in every row. The definition of the cost measure derives from the observation that a suitable rotation of the eigenvectors also defines a transformation of the graph Laplacian into a block-diagonal form. In this form, the connected components in the graph represent the clustering structure. 
STSC then chooses the largest number $k$ obtaining minimal costs, from a set of considered numbers of clusters. 

The definition of a cost measure may also rely on statistical properties of the dataset.  
Tibshirani et al.\@ deliver the statistical foundations for the elbow method with the gap statistic~\cite{tibshirani2001estimating}. Given a reference distribution, the gap statistic chooses the value of $k$ for which the gap between the approximation error and its 
expected value 
is the largest. The expected value is estimated by sampling the data of the reference distribution, and computing a clustering for every sampled dataset and every setting of $k$. 

The score-based methods cannot deliver a guarantee over the quality of the gained model. This is where statistical methods come into play, whose decisions over the number of clusters are based on statistical tests.
\textsc{GMeans}~\cite{hamerly2004learning} performs statistical tests for the hypothesis that 
the points in one cluster 
are Gaussian distributed. \textsc{PGMeans}~\cite{feng2007pg} improves over \textsc{GMeans}, by applying the Kolmogorov-Smirnov test for the goodness of fit between one-dimensional random projections of the data and the Gaussian mixture model. They empirically show that this approach is also suitable for non-Gaussian data.
An alternative to the Normality assumption is to assume that every cluster follows a unimodal distribution in a suitable space, which can be validated by the dip test. \textsc{DipMeans} provides a wrapper for $k$-means-related algorithms, testing for individual data points whether the distances to other points follow a unimodal distribution~\cite{2012kalogeratos}. Maurus and Plant argue that this approach is sensitive to noise and propose the algorithm \textsc{SkinnyDip}, focusing on scenarios with high background noise~\cite{maurus2016skinny}. Here, the authors assume that a basis transformation exists such that the clusters form a unimodal shape in all coordinate directions. All these approaches require a data transformation or projection, in order to apply the one-dimensional tests.
\section{A Nonparametric Bound}
We propose a bound on the probability that a specific pair of clusters is 
generated by a single cluster distribution.
Our bound relies on concentration inequalities, which have as input the mean and variance of the unified cluster distribution, which are easy to estimate. No assumptions on the distribution shape (e.g., Gaussian) must be made, and no projection is required. The core concentration inequality which we employ is the matrix Bernstein inequality.  
\begin{theorem}[{Matrix Bernstein \cite[Theorem 1.4]{tropp2012user}}]\label{thm:matrixBernst}
Consider a sequence of independent, random, symmetric matrices $A_i\in \mathbb{R}^{m\times m}$ for $1\leq i \leq n$. Assume that each random matrix satisfies:
\begin{equation*}
    \mathbb{E}[A_i] = \mathbf{0}\qquad \text{ and }\qquad \|A_i\|_{op}\leq \nu \text{ almost surely,}
\end{equation*}
and set 
$
    \sigma^2=\left\|\sum_{i}\mathbb{E}[A_i^2]\right\|_{op}.
$
Then, for all $t\geq 0$:
\begin{equation*}
    \mathbb{P}\left(\left\|\sum_iA_i\right\|_{op}\geq t\right)\leq m\exp\left(-\frac12\frac{t^2}{\sigma^2+\nu t/3}\right).
\end{equation*}
\end{theorem}
The matrix Bernstein bound employs the operator norm. For real-valued, symmetric matrices this equals the maximum eigenvalue in magnitude:
\begin{equation}
    \|A\|_{op} = \sup_{\|x\|= 1}\|Ax\|= \max_{1\leq j\leq m}|\lambda_j(A)|= \max_{x\in\mathbb{R}^{m}}R(A,x).\label{eq:opnorm}
\end{equation}
The relationship to the Rayleigh coefficient is important. This relationship is easy to derive by substituting $A$ with its eigendecomposition. We derive the following central result for the product matrix of centered random matrices.
\begin{theorem}[$ZZ$ Top Bound]\label{thm:sphericalcow}
Let $Z_{\cdot i}\in\mathbb{R}^{m}$ be independent samples of a random vector with mean zero, such that $\|Z_{\cdot i}\|\leq 1$ for $1\leq i\leq n$. Further, assume that $\mathbb{E}[Z_{ji}Z_{li}]=0$ for $j\neq l$ and $\mathbb{E}[Z_{ji}^2]=\sigma^2$ for $0<\sigma^2<1$ and $1\leq j\leq m$. Then, for $t>0$:
\begin{equation*}
    \mathbb{P}\left(\left\|ZZ^\top - n\sigma^2I\right\|_{op} \geq t\right) \leq m\exp\left(-\frac{1}{2}\cdot\frac{t^2}{n\sigma^2+t/3}\right).
\end{equation*}
\end{theorem}
\begin{proof}
We apply the matrix Bernstein inequality (Theorem~\ref{thm:matrixBernst}) to the sum of random matrices:
\[
    ZZ^\top - n\sigma^2 I = \sum_{i} \left(Z_{\cdot i}Z_{\cdot i}^\top - \sigma^2I\right).
\]
Assuming that the expected values satisfy $\mathbb{E}[Z_{ji}Z_{li}]=0$ for $j\neq l$ and $\mathbb{E}[Z_{ji}Z_{ji}]=\sigma^2$, 
yields that the random matrix $A_i = Z_{\cdot i}Z_{\cdot i}^\top - \sigma^2$ has mean zero:
$$
    \mathbb{E}\left[A_i\right] = \mathbb{E}\left[Z_{\cdot i}Z_{\cdot i}^\top - \sigma^2I\right] = \left(\mathbb{E}\left[Z_{j i}Z_{l i}\right]\right)_{jl} - \sigma^2I =0.
$$
The operator norm of the random matrices $A_i$ is bounded by one:
\begin{align*}
    \left\|A_i\right\|_{op} &= \left\|Z_{\cdot i}Z_{\cdot i}^\top - \sigma^2I\right\|_{op} = \sup_{\|x\|=1}x^\top \left(Z_{\cdot i}Z_{\cdot i}^\top-\sigma^2 I\right)x = \sup_{\|x\|=1}\left(Z_{\cdot i}^\top x\right)^2-\sigma^2\\ 
    &\leq 1-\sigma^2 \leq 1.
\end{align*}
The expected value of $A_i^2$ is:
$$
    \mathbb{E}\left[\left(Z_{\cdot i}Z_{\cdot i}^\top-\sigma^2I\right)\left(Z_{\cdot i}Z_{\cdot i}^\top-\sigma^2I\right)\right] \leq \mathbb{E}\left[Z_{\cdot i}Z_{\cdot i}^\top-2\sigma^2Z_{\cdot i}Z_{\cdot i}^\top + \sigma^4I\right]=\sigma^2I.
$$
Thus, the norm of the total variance is:
$$
    \left\|\sum_i\mathbb{E}\left[A_iA_i\right]\right\|_{op} = \left\|n \sigma^2I\right\|_{op} = n\sigma^2.
$$
The matrix Bernstein inequality then yields the result.\qed
\end{proof}
Applying Equation~(\ref{eq:opnorm}) to Theorem~\ref{thm:sphericalcow} yields the following corollary.
\begin{corollary}\label{cor:blimey}
Let $Z_{\cdot i}\in\mathbb{R}^{m}$ be independent samples of a random vector with mean zero, such that $\|Z_{\cdot i}\|\leq 1$ for $1\leq i\leq n$. Further assume that $\mathbb{E}[Z_{\cdot i}Z_{\cdot i}^\top]=\sigma^2I$ for $0<\sigma^2<1$ and $1\leq i\leq n$. 
Let $y\in\{0,1\}^{m}$ be an indicator vector of a cluster candidate, and denote:
\begin{align}
t=R\left(ZZ^\top,y\right) -n\sigma^2. \label{eq:t_Z}
\end{align}
Then, the probability that an indicator vector $y^*\in\{0,1\}^m$ exists with a Rayleigh coefficient 
such that 
$R(ZZ^\top,y)\leq R(ZZ^\top,y^*)$, 
is bounded as follows:
$$
    \mathbb{P}\left(\max_{y^*\in\{0,1\}^m}R\left(ZZ^\top,y^*\right) - n\sigma^2 \geq t\right) \leq m\exp\left(-\frac{1}{2}\cdot\frac{t^2}{n\sigma^2+t/3}\right).
$$
\end{corollary}
In practice, we must estimate the mean and variance of a candidate cluster. In this case, the relationship between the Rayleigh coefficient for the centered random matrix $Z$ and the original data matrix is specified as follows (assuming the data matrix $D$ is reduced to the observations belonging to a single cluster). 
\begin{remark}\label{rem:centeredRM}
Assume we want to bound the probability that two clusters indicated by $y,\thickbar{y}\in\{0,1\}^m$ are parts of one unified cluster represented by $D\in\mathbb{R}^{m\times n}$. 
We denote with $\mu = \frac1mD^\top \mathbf{1}$ 
the vector of sample means over the columns of $D$. The Rayleigh coefficient of $y$ with respect to the columnwise centered matrix $Z=D-\mathbf{1}\mu^\top$ is equal to (see Appendix~\ref{app:rem1} for full derivation):
$$
R\left(ZZ^\top,y\right) = \frac{|\thickbar{y}|}{m}\left(\frac{|\thickbar{y}|}{m}R\left(DD^\top,y\right)-\frac{|y|}{m}Cut\left(DD^\top,[y\ \thickbar{y}]\right)+\frac{|y|}{m}R\left(DD^\top,\thickbar{y}\right)\right)
$$
\end{remark}
The higher $R\left(ZZ^\top,y\right)$ is, the higher $t$ is in Equation \eqref{eq:t_Z}, and the lower the probability is that $y$ indicates a subset of the cluster represented by $D$.
Remark~\ref{rem:centeredRM} shows that the probability of $y$ indicating a subset of the larger cluster $D$, is determined by three things: the similarity within each of the candidate clusters ($R(DD^\top,y)$ and $R(DD^\top,\thickbar{y})$), the rational cut between these clusters ($Cut\left(DD^\top,[y\ \thickbar{y}]\right)$), and the ratio of points belonging to the one versus the other cluster ($|y|$ versus $|\thickbar{y}|$). As a result, the $ZZ$ Top Bound provides a natural balance of the within- and between-cluster similarity governing acceptance or rejection of a given clustering. 
\subsection{A Strategy to Find a Suitable Number of Clusters}
Remark~\ref{rem:centeredRM} and Corollary~\ref{cor:blimey} provide a method to bound the probability that two clusters are generated by the same distribution. Let us go through an example to discuss how we can employ the proposed bounds in a practical setting. Imagine two clusterings, the one employing a larger cluster covering the records denoted by the index set $\mathcal{J}\subset\{1,\ldots, m\}$, having center $\mu\in\mathbb{R}^n$; the other containing a subset of $\mathcal{J}$, indicated by $y$. Assume the following.
\begin{assumption}\label{ass}
If the indices $\mathcal{J}$ form a \emph{true} cluster, then the columns $D_{\mathcal{J} i}$ are independent samples of a random vector with mean $\mu_i\mathbf{1}\in\mathbb{R}^{|\mathcal{J}|}$, for $\mu_i\in\mathbb{R}$.
\end{assumption}
We then define the $|\mathcal{J}|$-dimensional scaled and centered sample vectors:
\[
Z_{\cdot i}= \frac{1}{\|D_{\mathcal{J} i}\|}(D_{\mathcal{J} i}-\mu_i\mathbf{1}) \quad \text{for all } 1\leq i\leq n.
\]  
Now, if one were to assume that $y\in\{0,1\}^m$ satisfies:
\begin{equation}\label{eq:testRayleigh}
    R(ZZ^\top,y)\geq \sqrt{2n\sigma^2\ln\left(\frac m\alpha\right)+\frac19\ln\left(\frac m\alpha\right)^2} +n\sigma^2 + \frac13\ln\left(\frac m\alpha\right),
\end{equation}
then Corollary~\ref{cor:blimey} implies that the probability of $\mathcal{J}$ being a true cluster and Equation~\eqref{eq:testRayleigh} holding is at most $\alpha$.
Hence, if $\alpha$ is small enough, then we conclude that $\mathcal{J}$ is not a true cluster; this conclusion is wrong only with the small probability $\alpha$ (which functions as a user-set significance level of a hypothesis test). 

Assumption \ref{ass} may not hold for all datasets. In particular, the assumption that the column vectors of the data matrix (comprising points from only one cluster) are sampled with the same variance parameter and a mean vector which is reducible to a scaled constant one vector, is not generally valid. Especially if the features of the dataset come from varying domains, the cluster assumptions may not hold. In this paper, we evaluate the $ZZ$ Top Bound in the scope of spectral clustering, where the feature domains are comparable; every feature corresponds to one eigenvector.
In particular, we consider a decomposition of the similarity matrix as shown in Equation~\eqref{eq:spectacl}. The rank of this decomposition is independent from the expected number of $k$, unlike in traditional spectral clustering algorithms. Hence, a factor matrix $D$ as computed in Equation~\eqref{eq:spectacl} can be treated like ordinary $k$-means input data.


\begin{algorithm}[t]
\caption{\textsc{SpecialK}($W, n, \alpha$)}
\begin{algorithmic}[1]
    \State $W\approx V^{(n)}\Lambda^{(n)}{V^{(n)}}^\top$ \Comment{Compute truncated eigendecomposition}
    \State $D_{ji}=\left|V_{ji}^{(n)}\right||\Lambda_{ii}|^{1/2}$ \Comment{For all $1\leq j\leq m$, $1\leq i\leq n$} \label{alg:computeD}
    \For {$k= 1,\ldots$}
        \State $Y^{(k)} \gets $\Call{k-means}{$D,k$} 
        \For {$c_1,c_2 \in \{1,\ldots,k\}, c_1>c_2$}
            \State $\mathcal{J}\gets \left\{j\middle| Y_{jc_1}^{(k)}+Y^{(k)}_{jc_2}>0\right\}$
            \State $\displaystyle Z_{\cdot i}\gets \frac{1}{\|D_{\mathcal{J} i}\|}\left(D_{\mathcal{J} i}-\frac{|D_{\mathcal{J} i}|}{|\mathcal{J}|}\mathbf{1}\right)$ \Comment{For all $j\in\mathcal{J}$}
            \State $\sigma^2\gets \frac{1}{n|\mathcal{J}|}\sum_{j,i}Z_{ji}^2$\Comment{Sample variance}
            \State $t\gets \max\left\{R(ZZ^\top, Y_{\mathcal{J}c})-n\sigma^2\middle| c\in\{c_1,c_2\}\right\}$
            \If {$\displaystyle|\mathcal{J}|\exp\left(-\frac{1}{2}\cdot \frac{t^2}{n\sigma^2+t/3}\right)>\alpha$}
                \State\textbf{return} $Y^{(k-1)}$
            \EndIf
        \EndFor
    \EndFor
\end{algorithmic}
\label{alg:specialK}
\end{algorithm}
We propose Algorithm~\ref{alg:specialK}, called \textsc{SpecialK}, since it provides a method for SPEctral Clustering to Infer the Appropriate Level $k$. Its input is a similarity matrix $W$, the feature dimensionality of the computed embedding, and the significance level $\alpha>0$. In the first two steps, the symmetric decomposition $W\approx DD^\top$ is computed. For an increasing number of clusters, we compute a $k$-means clustering. For every pair of clusters, we compute the probability that both clusters are actually subsets of the unified cluster. If this probability is larger than the significance level $\alpha$, then we conclude that the current clustering splits an actual cluster into two and we return the previous model. 
\section{Experiments}
In comparative experiments, our state-of-the-art competitors are 
Self-Tuning Spectral Clustering (STSC) \cite{zelnik2005self} and \textsc{PGMeans} \cite{feng2007pg}, whose implementations have been kindly provided by the authors. Since we strive for applicability on nonconvex cluster shapes, we apply \textsc{PGMeans} to the projected eigenvectors (as computed in Line~\ref{alg:computeD} in Algorithm~\ref{alg:specialK}) of a given similarity matrix. We set for \textsc{PGMeans} the significance level for every test of  12 random projections to $\alpha=0.01/12$. 
We also included \textsc{SkinnyDip} \cite{maurus2016skinny} in our evaluation. However, applying this algorithm on the decomposition matrix results in a vast number of returned clusters ($\hat{k}\approx 50$ while the actual value is $k\leq 3$) where most of the data points are attributed to noise. Since this result is clearly wrong, we eschew further comparison with this algorithm.

We consider two variants of similarity matrices: $W_R$ and $W_C$. The former employs the $\epsilon$-neighborhood adjacency matrix, where $\epsilon$ is set such that $99\%$ of the data points have at least ten neighbors; the latter employs the symmetrically normalized adjacency matrix of the $k$NN graph. 
STSC computes its own similarity matrix for a given set of points. To do so, it requires a number of considered neighbors, which we set to the default value of 15. Note that the result of STSC comes with its own quality measurement for the computed models; higher is better. We provide a Python implementation of \textsc{SpecialK}\footnote{\url{https://github.com/Sibylse/SpecialK}}. In this implementation, we do not assess for all possible pairs of clusters if they emerge from one distribution, but only for the ten cluster pairs having the highest cut.

\subsection{Synthetic Experiments}
\begin{figure}[t!]
\centering
\includegraphics[width=\textwidth]{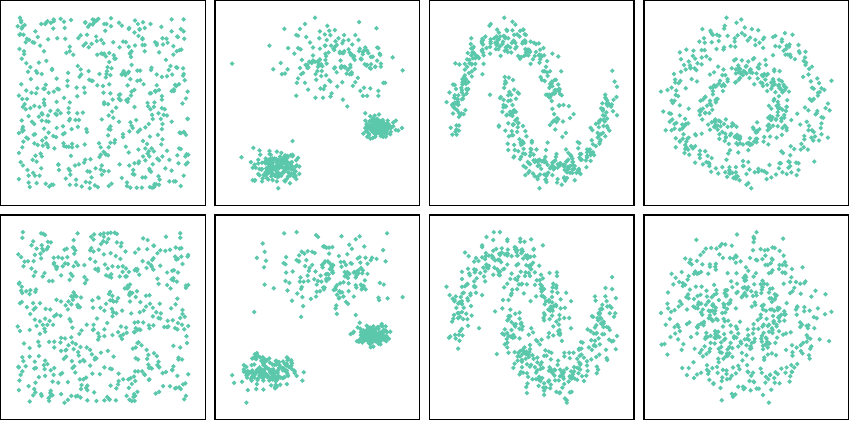}
\caption{Visualization of the datasets (from left to right: random, three blobs, two moons, and two circles) with noise parameters equal to 0.1 (top row) and 0.15 (bottom row).}
\label{fig:vizSynth}
\end{figure}
We generate benchmark datasets, using the scikit library for Python. 
Figure~\ref{fig:vizSynth} provides examples of the generated datasets, which come in four types of seeded cluster shape (random, blobs, moons, and circles), and a noise parameter (set to $0.1$ and $0.15$, respectively, in the figure).   
For each shape and noise specification, we generate $m=1500$ data points. The noise is Gaussian, as provided by the scikit noise parameter (cf.\@ \url{http://scikit-learn.org}). This parameter takes a numeric value, for which we investigate ten settings: we traverse the range $[0,0.225]$ by increments of size $0.025$.  For every shape and noise setting, we generate five datasets.
Unless otherwise specified, we employ a dimensionality of $n=200$ as parameter for \textsc{SpecialK}. However, we use a different rank of $n=50$ for \textsc{PGMeans}, whose results benefit from this setting. We set \textsc{SpecialK}'s significance level to $\alpha=0.01$. For all algorithms, we consider only values of $k\in\{1,\ldots,5\}$; Figure \ref{fig:vizSynth} illustrates that higher values are clearly nonsense.
\begin{figure}[t!]
\centering
\input{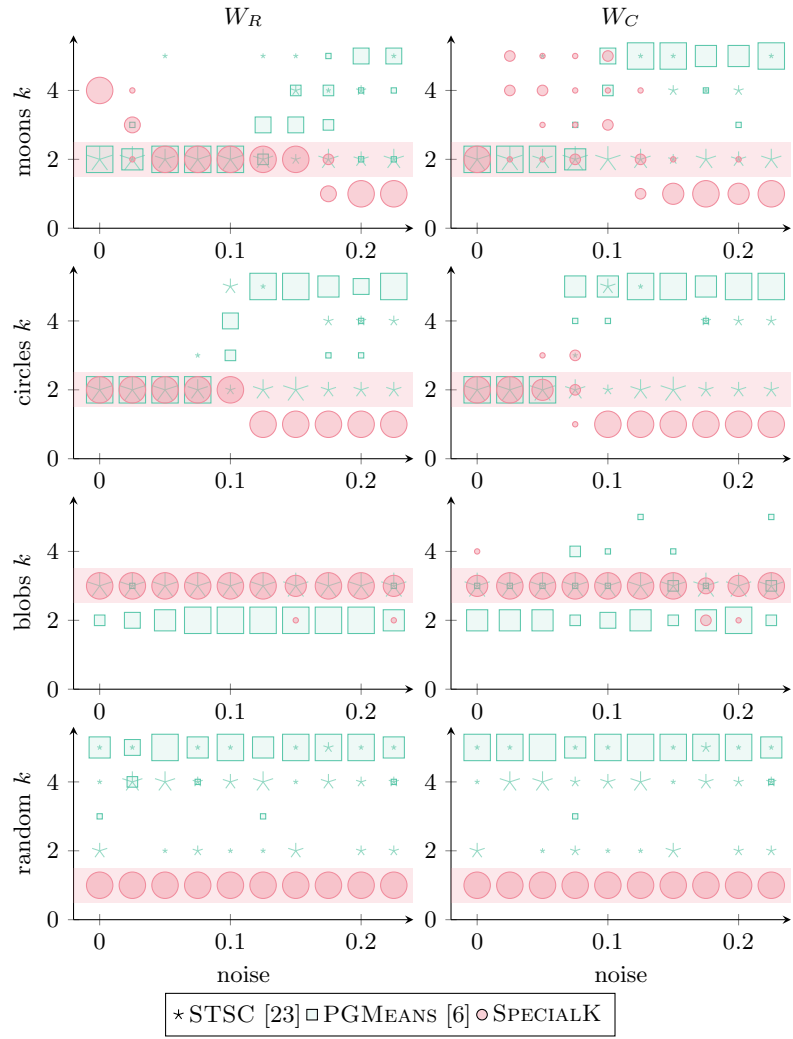} 

\pgfplotsset{
scatSPHSty/.style={scatter, only marks , fill opacity=0.4, fill=cSPH, 
     scatter/use mapped color={
                draw=cSPH,
                fill=cSPH,
            },scatter src=explicit,
     visualization depends on ={\thisrow{z} \as \perpointmarksize},
     scatter/@pre marker code/.append style={/tikz/mark size=1*\perpointmarksize, /tikz/mark color=cCSPH}},
scatSTSCSty/.style={scatter, only marks,fill opacity=0.1, fill=cSTSC, mark=star,
     scatter/use mapped color={
                draw=cSTSC,
                fill=cSTSC,
            },scatter src=explicit,
     visualization depends on ={\thisrow{z} \as \perpointmarksize},
     scatter/@pre marker code/.append style={/tikz/mark size=1*\perpointmarksize}},
scatPGSty/.style={scatter, only marks,fill opacity=0.1, fill=cPG, mark=square*,
     scatter/use mapped color={
                draw=cPG,
                fill=cPG,
            },scatter src=explicit,
     visualization depends on ={\thisrow{z} \as \perpointmarksize},
     scatter/@pre marker code/.append style={/tikz/mark size=1*\perpointmarksize}},
}

\begin{tikzpicture}
\begin{groupplot}[group style={group size= 2 by 4, horizontal sep=0.5cm, vertical sep=.5cm},
    	height=.34\textwidth,
    	width=.5\textwidth,
    	legend columns =-1,
        title style={yshift=-1.5ex}, 
        axis lines = left, xmin=-0.02, xmax=0.24]
 	\nextgroupplot[title={$W_R$},
 	ymin=0,ymax=5.6,ylabel={moons $k$} ,
    legend to name=zelda]
        \draw[fill=cSPH, draw opacity=0, fill opacity=0.2] ({axis cs:-0.02,1.5}) rectangle ({axis cs:0.24,2.5});
	    \addplot[scatSTSCSty] table[meta=z]{moonsR_RLS.dat};
	    \addlegendentry{\textsc{STSC} \cite{zelnik2005self}};
	    \addplot[scatPGSty] table[meta=z]{moonsR_PGMeans.dat};
	    \addlegendentry{\textsc{PGMeans} \cite{feng2007pg}};
	    \addplot[scatSPHSty] table[meta=z]{moonsR_rankSph.dat};
 	    \addlegendentry{\textsc{SpecialK}};
    \nextgroupplot[title={$W_C$},
    ymin=0,ymax=5.6]
        \draw[fill=cSPH, draw opacity=0, fill opacity=0.2] ({axis cs:-0.02,1.5}) rectangle ({axis cs:0.24,2.5});
	    \addplot[scatSTSCSty] table[meta=z]{moonsN_RLS.dat};
	    \addplot[scatPGSty] table[meta=z]{moonsN_PGMeans.dat};
	    \addplot[scatSPHSty] table[meta=z]{moonsN_rankSph.dat};
    \nextgroupplot[ylabel={circles $k$},
    ymin=0,ymax=5.6]
        \draw[fill=cSPH, draw opacity=0, fill opacity=0.2] ({axis cs:-0.02,1.5}) rectangle ({axis cs:0.24,2.5});
	    \addplot[scatSTSCSty] table[meta=z]{circlesR_RLS.dat};
	    \addplot[scatPGSty] table[meta=z]{circlesR_PGMeans.dat};
	    \addplot[scatSPHSty] table[meta=z]{circlesR_rankSph.dat};
    \nextgroupplot[ymin=0,ymax=5.6] 
        \draw[fill=cSPH, draw opacity=0, fill opacity=0.2] ({axis cs:-0.02,1.5}) rectangle ({axis cs:0.24,2.5});
	    \addplot[scatSTSCSty] table[meta=z]{circlesN_RLS.dat};
	    \addplot[scatPGSty] table[meta=z]{circlesN_PGMeans.dat};
	    \addplot[scatSPHSty] table[meta=z]{circlesN_rankSph.dat};
    \nextgroupplot[ylabel={blobs $k$},
    ymin=0,ymax=5.6] 
        \draw[fill=cSPH, draw opacity=0, fill opacity=0.2] ({axis cs:-0.02,2.5}) rectangle ({axis cs:0.24,3.5});
	    \addplot[scatSTSCSty] table[meta=z]{blobsR_RLS.dat};
	    \addplot[scatPGSty] table[meta=z]{blobsR_PGMeans.dat};
	    \addplot[scatSPHSty] table[meta=z]{blobsR_rankSph.dat};
    \nextgroupplot[ymin=0,ymax=5.6] 
        \draw[fill=cSPH, draw opacity=0, fill opacity=0.2] ({axis cs:-0.02,2.5}) rectangle ({axis cs:0.24,3.5});
	    \addplot[scatSTSCSty] table[meta=z]{blobsN_RLS.dat};
	    \addplot[scatPGSty] table[meta=z]{blobsN_PGMeans.dat};
	    \addplot[scatSPHSty] table[meta=z]{blobsN_rankSph.dat};
    \nextgroupplot[ymin=0,ymax=5.6,ylabel={random  $k$},xlabel={noise}] 
        \draw[fill=cSPH, draw opacity=0, fill opacity=0.2] ({axis cs:-0.02,0.5}) rectangle ({axis cs:0.24,1.5});
        \addplot[scatSPHSty] table[meta=z]{randomR_rankSph.dat};
	    \addplot[scatSTSCSty] table[meta=z]{randomR_RLS.dat};
	    \addplot[scatPGSty] table[meta=z]{randomR_PGMeans.dat};
    \nextgroupplot[ymin=0,ymax=5.6,xlabel={noise}] 
        \draw[fill=cSPH, draw opacity=0, fill opacity=0.2] ({axis cs:-0.02,0.5}) rectangle ({axis cs:0.24,1.5});
        \addplot[scatSPHSty] table[meta=z]{randomN_rankSph.dat};
	    \addplot[scatSTSCSty] table[meta=z]{randomN_RLS.dat};
	    \addplot[scatPGSty] table[meta=z]{randomN_PGMeans.dat};
\end{groupplot}
\end{tikzpicture}\\

\pgfplotslegendfromname{zelda}
\caption{Variation of noise, comparison of the derived number of clusters for the two moons, two circles, three blobs, and random datasets.}
\label{fig:noiseSynth}
\end{figure}

In Figure~\ref{fig:noiseSynth}, we plot every method's estimated number of clusters for the four datasets (rows) and two similarity matrices (columns) $W_R$ and $W_C$ (since STSC employs its own similarity matrix, the plot with respect to STSC does not vary for these settings: STSC behaves exactly the same in left-column and right-column subplots on the same dataset). In every subplot, the $x$-axis denotes the setting for the noise parameter.  On the $y$-axis we aggregate the number of clusters detected by each of the three competitors; the correct number of clusters ($3$ for blobs, $2$ for moons and circles, and $1$ for random) for each subplot is highlighted by a pink band.  Every column in every subplot corresponds to a single setting of shape and noise; recall that we generated five version of the dataset for such a setting.  The column now gathers for each of the three competitors (marked in various shapes; see the legend of the figure) which setting of $k$ it determines to be true how often out of the five times.  This frequency is represented by mark size: for instance, if \textsc{PGMeans} determines five distinct values of $k$ for the five datasets, we get five small squares in the column, but if it determines the same value of $k$ for all five datasets, we get one big square in the column.  An algorithm performs well if many of its big marks fall in the highlighted band.

Figure~\ref{fig:noiseSynth} illustrates that \textsc{PGMeans} is all over the place.  For the moons and circles datasets, it correctly identifies the number of clusters at low noise levels, but from a certain noise level onwards, it substantially overestimates $k$. On the moons dataset under $W_R$ this behavior is subtle; under $W_C$ and on the circles dataset the jump from $2$ to $5$ clusters is jarring.  On the blobs dataset, which really isn't that difficult a task, \textsc{PGMeans} systematically underestimates $k$. STSC, on the other hand, does quite well.  It doesn't make a single mistake on the blobs dataset. STSC generally has the right idea on the circles and moons datasets: at low noise levels, it correctly determines $k$, and at higher noise levels, it alternates between the correct number of clusters and an overestimation. Conversely, \textsc{SpecialK} has a tendency to err on the other side, if at all.  On the circles dataset, it correctly identifies the number of clusters at low noise levels, and packs all observations into a single cluster at high noise levels, which, visually, looks about right (cf.\@ Figure \ref{fig:vizSynth}, lower right). On the blobs dataset, \textsc{SpecialK} generally finds the right level of $k$, only making incidental mistakes.  Performance seems more erratic on the two moons dataset, especially with the $W_C$ similarity matrix. Similarity matrices to this dataset exhibit unusual effects, which are due to the symmetry of the two clusters \cite{hess2019spectacl}.  To counter these effects, we discard all eigenvectors which are extremely correlated, which we define as having an absolute Spearman rank correlation $|\rho|>0.95$.  Subsequently, the rank is correctly estimated until the noise makes the two clusters appear as one. 

The bottom row of Figure \ref{fig:noiseSynth} is quite revealing.  It illustrates how both STSC and \textsc{PGMeans} are prone to overfitting.  On data that is pure noise, both these methods tend to find several clusters anyway, despite there being no natural grouping in the dataset: it really rather is one monolithic whole.  \textsc{SpecialK} is the only algorithm capable of identifying this, and it does so without a single mistake.  Oddly, STSC seems to favor an even number of clusters in random data. \textsc{PGMeans} tends to favor defining as many clusters as possible on random data.  

Empirical results on the sensitivity to the input parameter $n$ (the employed number of eigenvectors) are given in Appendix~\ref{app:sensitivity_n}.
\subsection{Real-World Data Experiments}\label{sec:expReal}
Experiments on synthetic datasets provide ample evidence that \textsc{PGMeans} cannot compete with STSC and \textsc{SpecialK}.  Hence, we conduct experiments with only those latter two algorithms on selected real-world datasets, whose characteristics are summarized in the left half of Table~\ref{tbl:datasets}. The Pulsar dataset\footnote{\url{https://www.kaggle.com/pavanraj159/predicting-pulsar-star-in-the-universe}} contains samples of Pulsar candidates, where the positive class of real Pulsar examples poses a minority against noise effects. The Sloan dataset\footnote{\url{https://www.kaggle.com/lucidlenn/sloan-digital-sky-survey}} comprises measurements of the Sloan Digital Sky Survey, where every observation belongs either to a star, a galaxy, or a quasar. The MNIST dataset~\cite{lecun1998gradient} 
is a well-known collection of handwritten numbers: the ten classes are the ten digits from zero to nine.
The HMNIST dataset~\cite{kather2016multi} comprises histology tiles from patients with colorectal cancer. The classes correspond to eight types of tissue. For these real-world datasets, we increase \textsc{SpecialK}'s parameter to $n=1000$, since the real-world datasets have at least three times as many examples as the synthetic datasets.
\begin{table}[t]
\centering
\caption{Experimental results on real-world datasets. The left half contains metadata on the employed datasets: names, numbers of rows ($m$) and columns ($d$), and the real number of classes in the data (Actual $k$). The right half contains the results of the experiments: the number of classes $k$ determined by the algorithms STSC and \textsc{SpecialK}, the latter parameterized with similarity matrices $W_R$ and $W_C$.}\label{tbl:datasets}
\begin{tabular}{lrrr|ccc}\toprule
    Dataset & $m$ & $d$ & Actual $k$ & \multicolumn{3}{c}{Determined $k$}\\
    &&&&STSC&\multicolumn{2}{c}{\textsc{SpecialK}}\\
    &&&&&$W_R$&$W_C$\\ \midrule
    Pulsar & 17\thinspace898 & 9 & 2 &2&2&2\\
    Sloan & 10\thinspace000 & 16 & 3 &4&4&2\\
    MNIST & 60\thinspace000 & 784 & 10 &2&2&3 \\ 
    HMNIST & 5\thinspace000 & 4\thinspace 096 & 8 &3&4&4\\\bottomrule
    \end{tabular}
\end{table}

\begin{table}[t]
    \centering
    \caption{NMI scores, probability bounds ($p$) and costs for the MNIST dataset. The selected rank and the corresponding NMI score is highlighted for every method.}
    \label{tbl:nmi_mnist}
    \begin{tabular}{r|cc|cc|cc}\toprule
    & \multicolumn{4}{c|}{\textsc{SpecialK}} & \multicolumn{2}{c}{STSC \cite{zelnik2005self}}\\
    & \multicolumn{2}{c|}{$W_C$} & \multicolumn{2}{c|}{$W_R$}\\
        $k$ & NMI & $p$ & NMI & $p$ & NMI & quality\\ \midrule      
        2&0.317&$10^{-6}$&\textbf{0.195}&$\mathbf{10^{-23}}$&\textbf{0.306}&\textbf{0.987}\\
        3&\textbf{0.518}&$\mathbf{10^{-4}}$&0.207&1.000&0.290&0.978\\
        4&0.668&0.019&0.244&1.000&0.282&0.969\\
        5&0.687&0.011&0.281&1.000&0.274&0.970\\
        6&0.759&0.004&0.294&1.000&0.271&0.970\\
        7&0.760&1.000&0.311&1.000&0.287&0.948\\
        8&0.759&1.000&0.333&1.000&0.279&0.954\\
        9&0.757&1.000&0.347&1.000&0.277&0.956\\
        10&0.756&1.000&0.350&1.000&0.297&0.942\\
        11&0.747&1.000&0.348&1.000&0.362&0.957\\
        \bottomrule
    \end{tabular}
\end{table}

Results of the procedure when we let $k$ simply increase up to eleven are given in Table \ref{tbl:nmi_mnist}.  By $p$ we denote the maximum of the probability bounds \textsc{SpecialK} computes, as outlined in line $10$ of Algorithm \ref{alg:specialK} and mirrored at the end of Corollary \ref{cor:blimey}. For STSC, we output the quality values on which the algorithm bases its decisions (higher is better).  Additionally, we give the Normalized Mutual Information (NMI) scores between the constructed clustering and the actual class labels, matched via the Hungarian algorithm \cite{1955Kuhn}; typically, higher is better. 

STSC returns the $k$ for which the quality column contains the highest value.  By Algorithm \ref{alg:specialK}, \textsc{SpecialK} returns the lowest $k$ for which its $p$-value is below $\alpha$, while the $p$-value for $k+1$ is above $\alpha$.  The selected values are highlighted in Table \ref{tbl:nmi_mnist}, and the determined values for $k$ are entered in the right half of Table \ref{tbl:datasets}.  

Across all datasets, the right half of Table \ref{tbl:datasets} gives the determined values for $k$.  All methods reconstruct the actual $k$ well on the Pulsar dataset, but none of the determined values for $k$ are equal to the actual value for $k$ on the other three datasets.  On Sloan, both methods are in the correct ballpark.  On HMNIST, \textsc{SpecialK} is closer to the right answer than STSC.  On MNIST, the true value of $k$ is $10$, but both methods determine a substantially lower number of clusters as appropriate.  One can get more information on the behavior of the algorithms by taking a closer look at Table \ref{tbl:nmi_mnist}; similar tables for the other datasets can be found in Appendix~\ref{app:expAddReal}.

For STSC, the highest NMI value is actually obtained for $k=11$, which is a too high number of clusters, but quite close to the actual $k$.  However, the computed quality does not mirror this finding.  Also, the NMI value for the actual $k=10$ is substantially lower than the NMI for $k=11$, and the NMI for $k=10$ is lower than the NMI value for the selected $k=2$.  Hence, NMI cannot just replace the quality in STSC.  For \textsc{SpecialK}, $p$-value behavior is unambiguous under $W_R$.  Notice that NMI peaks at the right spot.  

Under $W_C$, things get more interesting.  While the $p$-value for $k=4$ indeed surpasses the threshold $\alpha$, a slightly less restrictive setting (in many corners of science, $\alpha=5\%$ is considered acceptable) would have changed the outcome to $k=6$.  At that stage, the $p$-value suddenly unambiguously shoots to $1$; more than $6$ clusters is definitely not acceptable.  We see this behavior persist through the other datasets as well: there is always a specific value of $k$, such that the $p$-values for all $k'>k$ are drastically higher than the $p$-values for all $k''\leq k$.  While Algorithm \ref{alg:specialK} makes an automated decision (and this is a desirable property of an algorithm; no post-hoc subjective decision is necessary), if an end-user wants to invest the time to look at the table and select the correct number of clusters themselves, the $p$-values give a clear and unambiguous direction to that decision.

The entire last paragraph glosses over the fact that the actual $k$ for MNIST is not $6$, but $10$.  In fact, at first sight, the right half of Table \ref{tbl:datasets} paints an unpleasant picture when compared to the Actual $k$ column in the left half. However, the correct conclusion is that that column label is misleading.  We have some real-world datasets with a given class label, providing a natural grouping of the data into $k$ clusters.  The task of clustering is also to find a natural grouping in the dataset.  Clustering, however, is not necessarily built to \emph{reconstruct any given} natural grouping: this task is unsupervised!  Hence, if an algorithm finds a natural group on the MNIST dataset of relatively bulbous digits, this is a rousing success in terms of the clustering task.  However, this group encompasses the digits $6$, $8$, and $9$ (and perhaps others), which reduces the cardinality of the resulting clustering when compared to the clustering that partitions all digits.  Therefore, no hard conclusions can be drawn from the determined $k$ not matching the actual $k$.  This is a cautionary tale (also pointed out in \cite{hess2019spectacl}), warning against a form of evaluation that is pervasive in clustering, but does not measure the level of success it ought to measure in a solution to the clustering task.

\section{Conclusions}

We propose a probability bound, that enables to make a hard, statistically founded decision on the question whether two clusters should be fused together or kept apart.  Given a significance level $\alpha$ (with the usual connotation and canonical value settings), this results in an algorithm for spectral clustering, automatically determining the appropriate number of clusters $k$.  Since it provides a method for SPEctral Clustering to Infer the Appropriate Level $k$, the algorithm is dubbed \textsc{SpecialK}.  Automatically determining $k$ in a statistically nonparametric manner for clusters with nonconvex shapes is, to the best of our knowledge, a novel contribution to data mining research.  Also, unlike existing algorithms, \textsc{SpecialK} can decide that the data encompasses only one cluster.

\textsc{SpecialK} is built to automatically make a decision on which $k$ to select, which it does by comparing subsequent $p$-values provided by the probability bound, and checking whether they undercut the significance level $\alpha$. As a consequence, the user can elect to simply be satisfied with whatever $k$ \textsc{SpecialK} provides.  In the experiments on the MNIST dataset, we have seen that the perfect setting of $k$ is in the eye of the beholder: several people can have several contrasting opinions of what constitutes a natural grouping of the data in a real-world setting.  In such a case, one can extract more meaningful information out of the results \textsc{SpecialK} provides, by looking into the table of NMI scores and $p$-values for a range of settings of $k$.  Eliciting meaning from this table is a subjective task.  However, in all such tables for all datasets we have seen so far, there is a clear watershed moment where $k$ gets too big, and relatively low $p$-values are followed by dramatically high $p$-values for any higher $k$.  Turning this soft observation into a hard procedure would be interesting future work.
%
%

%
%
%
\bibliographystyle{splncs04}

\begin{appendix}
\section{Full Derivation of Remark 1}\label{app:rem1}

\begin{remark}\label{rem:centeredRM}
Assume we want to bound the probability that two clusters indicated by $y,\thickbar{y}\in\{0,1\}^m$ are parts of one unified cluster represented by $D\in\mathbb{R}^{m\times n}$. 
We denote with $\mu = \frac1mD^\top \mathbf{1}$ 
the vector of sample means over the columns of $D$. The Rayleigh coefficient of $y$ with respect to the columnwise centered matrix $Z=D-\mathbf{1}\mu^\top$ is equal to:
\begin{align*}
    R\left(ZZ^\top,y\right) &= \frac{y^\top DD^\top y}{|y|} -2 \frac{y^\top \mathbf{1}\mu^\top D^\top y}{|y|} + \frac{y^\top \mathbf{1}\mu^\top \mu \mathbf{1}^\top y}{|y|}\\
    &= R(DD^\top,y) -2\frac{\mathbf{1}^\top D D^\top y}{m} + \frac{|y|}{m}\frac{\mathbf{1}^\top DD^\top \mathbf{1}}{m}.
\end{align*}
We replace the constant one vector with $\mathbf{1}=y+\thickbar{y}$, rewriting the second and third term as:
\begin{align*}
    \frac{\mathbf{1}^\top DD^\top\mathbf{1}}{m}
    =& \frac{|y|}{m}R\left(DD^\top,y\right)+\frac{|y|}{m}C\left(DD^\top,y\right)\\&\phantom{\frac{|y|}{m}R\left(DD^\top,y\right)}+\frac{|\thickbar{y}|}{m}C\left(DD^\top,\thickbar{y}\right)+ \frac{|\thickbar{y}|}{m}R\left(DD^\top,\thickbar{y}\right)\\
    -2\frac{\mathbf{1}^\top DD^\top y}{m} =& 
    -2\frac{|y|}{m}R\left(DD^\top,y\right) - \frac{|y|}{m}C\left(DD^\top,y\right) - \frac{|\thickbar{y}|}{m}C\left(DD^\top,\thickbar{y}\right). 
\end{align*}
\begin{align*}
    R\left(ZZ^\top,y\right) =& \left(1-2\frac{|y|}{m}+\frac{|y|^2}{m^2}\right)R\left(DD^\top,y\right) -\frac{|y|}{m}\left(1-\frac{|y|}{m}\right)C\left(DD^\top,y\right)\\
    &-\frac{|\thickbar{y}|}{m}\left(1-\frac{|y|}{m}\right)C\left(DD^\top,\thickbar{y}\right)+\frac{|y|}{m}\frac{|\thickbar{y}|}{m}R\left(DD^\top,\thickbar{y}\right)\\
    =&\left(1-\frac{|y|}{m}\right)^2R\left(DD^\top,y\right)- \frac{|y|}{m}\frac{|\thickbar{y}|}{m}C\left(DD^\top,y\right)\\&-\frac{|y|}{m}\frac{|\thickbar{y}|}{m}C\left(DD^\top,\thickbar{y}\right)+ \frac{|y|}{m}\frac{|\thickbar{y}|}{m}R\left(DD^\top,\thickbar{y}\right)\\
    =&\frac{|\thickbar{y}|}{m}\left(\frac{|\thickbar{y}|}{m}R\left(DD^\top,y\right)-\frac{|y|}{m}Cut\left(DD^\top,[y\ \thickbar{y}]\right)+\frac{|y|}{m}R\left(DD^\top,\thickbar{y}\right)\right)
\end{align*}
\end{remark}

\section{Sensitivity to Parameter $n$}\label{app:sensitivity_n}

\begin{figure}[h]
\centering
\input{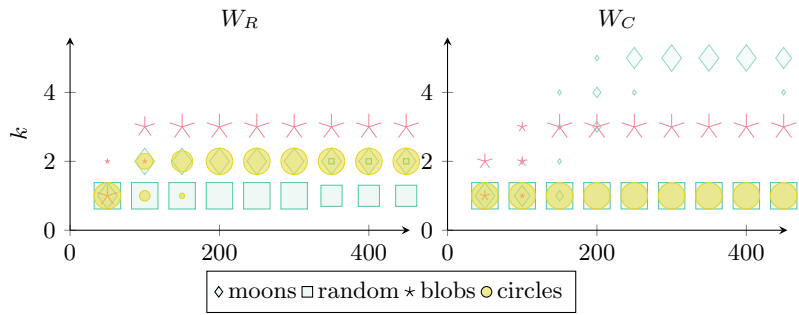}
\caption{Variation of the number of employed eigenvectors $n$, comparison of the derived number of clusters for the two moons, two circles, three blobs, and random datasets.}
\label{fig:dSynth}
\end{figure}

Figure \ref{fig:dSynth} displays the effect of increasing the number of employed eigenvectors on the number of clusters determined by \textsc{SpecialK}, with similarity matrix $W_R$ (left plot) and $W_C$ (right plot). Recall the correct numbers of clusters: $k=1$ for the random dataset, $k=2$ for the moons and circles dataset, and $k=3$ for the blobs dataset. Here, the noise parameter is set to $0.1$. With an increasing number of eigenvectors, the algorithm stabilizes in its insistence on a certain choice of $k$ (albeit the right choice in only six out of the eight cases, with underestimation for circles under $W_C$ and overestimation for moons under $W_C$).

\section{Additional Real-World Experimental Results}\label{app:expAddReal}

We provide tables analogous to \cite[Table 2]{2019Hess} for the remaining real-world data experiments.  This information can be found for the Pulsar dataset in Table \ref{tbl:nmi_pulsar}, for Sloan in Table \ref{tbl:nmi_sloan}, and for HMNIST in Table \ref{tbl:nmi_hmnist}.

\begin{table}[t]
    \centering
    \caption{NMI scores, probability bounds ($p$) and costs for the Pulsar dataset. The selected rank and the corresponding NMI score is highlighted for every method.}
    \label{tbl:nmi_pulsar}
    \begin{tabular}{r|cc|cc|cc}\toprule
    & \multicolumn{4}{c|}{\textsc{SpecialK}} & \multicolumn{2}{c}{STSC \cite{zelnik2005self}}\\
    & \multicolumn{2}{c|}{$W_C$} & \multicolumn{2}{c|}{$W_R$}\\
        $k$ & NMI & $p$ & NMI & $p$ & NMI & quality\\ \midrule      
        2&\textbf{0.014}&$\mathbf{10^{-11}}$ &\textbf{0.120}&$\mathbf{10^{-11}}$&\textbf{0.006}&\textbf{0.996}\\
        3&0.134&1.000&0.123&1.000&0.012&0.994\\
        4&0.130&1.000&0.112&1.000&0.017&0.988\\
        5&0.158&1.000&0.131&1.000&0.024&0.978\\
        6&0.149&1.000&0.126&1.000&0.021&0.961\\
        7&0.108&1.000&0.119&1.000&0.177&0.949\\
        8&0.106&1.000&0.115&1.000&0.192&0.963\\
        9&0.113&1.000&0.113&1.000&0.188&0.966\\
        10&0.111&1.000&0.132&1.000&0.161&0.967\\\bottomrule
    \end{tabular}
\end{table}
\begin{table}[t]
    \centering
    \caption{NMI scores, probability bounds ($p$) and costs for the Sloan dataset. The selected rank and the corresponding NMI score is highlighted for every method.}
    \label{tbl:nmi_sloan}
    \begin{tabular}{r|cc|cc|cc}\toprule
    & \multicolumn{4}{c|}{\textsc{SpecialK}} & \multicolumn{2}{c}{STSC \cite{zelnik2005self}}\\
    & \multicolumn{2}{c|}{$W_C$} & \multicolumn{2}{c|}{$W_R$}\\
        $k$ & NMI & $p$ & NMI & $p$ & NMI & quality\\ \midrule      
        2&\textbf{0.009}&\textbf{0.004}&0.093&$10^{-40}$&0.395&0.999\\
        3&0.062&1.000&0.224&$10^{-14}$&0.364&0.999\\
        4&0.062&1.000&\textbf{0.219}&$\mathbf{10^{-14}}$&\textbf{0.359}&\textbf{0.999}\\
        5&0.069&1.000&0.249&1.000&0.346&0.998\\
        6&0.064&1.000&0.209&1.000&0.341&0.997\\
        7&0.071&1.000&0.227&1.000&0.339&0.997\\
        8&0.071&1.000&0.228&1.000&0.336&0.997\\
        9&0.068&1.000&0.224&1.000&0.335&0.997\\
        10&0.069&1.000&0.219&1.000&0.314&0.988
\\\bottomrule
    \end{tabular}
\end{table}
\begin{table}[t]
    \centering
    \caption{NMI scores, probability bounds ($p$) and costs for the HMNIST dataset. The selected rank and the corresponding NMI score is highlighted for every method.}
    \label{tbl:nmi_hmnist}
    \begin{tabular}{r|cc|cc|cc}\toprule
    & \multicolumn{4}{c|}{\textsc{SpecialK}} & \multicolumn{2}{c}{STSC \cite{zelnik2005self}}\\
    & \multicolumn{2}{c|}{$W_C$} & \multicolumn{2}{c|}{$W_R$}\\
        $k$ & NMI & $p$ & NMI & $p$ & NMI & quality\\ \midrule      
        2&0.326&$10^{-28}$&0.449&$10^{-34}$&0.375&0.964\\
        3&0.291&$10^{-17}$&0.369&0.262&\textbf{0.528}&\textbf{0.990}\\
        4&\textbf{0.452}&$\mathbf{10^{-8}}$&\textbf{0.376}&\textbf{0.060}&0.511&0.962\\
        5&0.437&0.307&0.393&1.000&0.492&0.949\\
        6&0.429&1.000&0.388&1.000&0.482&0.941\\
        7&0.417&1.000&0.391&1.000&0.469&0.938\\
        8&0.414&1.000&0.403&1.000&0.487&0.957\\
        9&0.419&1.000&0.404&1.000&0.476&0.944\\
        10&0.439&1.000&0.417&1.000&0.475&0.946\\
        11&0.404&1.000&0.435&1.000&0.479&0.949\\\bottomrule
    \end{tabular}
\end{table}

To augment the result analyzed in the penultimate paragraph of Section~\ref{sec:expReal}, we give here a short discussion of the found clustering.  Recall that we analyze the MNIST dataset, partitioning the handwritten digits with the \textsc{SpecialK} algorithm under similarity matrix $W_C$.  If our significance level $\alpha$ would have been set to the not uncommon $0.05$ instead of our $0.01$, we would have found a clustering into $k=6$ clusters.  These clusters would roughly correspond to the following partitioning of digits. First cluster: 4, 7, 9. Second cluster: 6, 5. Third cluster: 8, 3 (plus some extra 9s and 5s). Fourth cluster: 0. Fifth cluster: 1. Sixth cluster: 2.  We see that the first cluster focuses mostly on those digits with a stroke along the antidiagonal of the picture.  The second and third cluster encompass two types of bulbous digits, and the remaining three clusters specialize in the remaining three characters.  This does not correspond to a complete partitioning of the ten digits, but one could reasonably defend that this clustering is indeed a natural grouping of the dataset.
\end{appendix}
\end{document}